\documentclass{article}

\usepackage{microtype}
\usepackage{graphicx}
\usepackage{subfigure}
\usepackage{booktabs} 

\usepackage{hyperref}

\usepackage[accepted]{icml2024}

\usepackage{amsmath}
\usepackage{amssymb}
\usepackage{mathtools}
\usepackage{amsthm}

\usepackage[capitalize,noabbrev]{cleveref}

\theoremstyle{plain}
\newtheorem{theorem}{Theorem}
\newtheorem{proposition}{Proposition}

\theoremstyle{definition}
\newtheorem{definition}{Definition}

\theoremstyle{remark}

\usepackage{afterpage}

\usepackage[textsize=tiny]{todonotes}

\icmltitlerunning{Adversarial Multi-dueling Bandits}

\usepackage{breqn}
\usepackage{verbatim}
\usepackage{dsfont}
\usepackage{caption}
\usepackage{thm-restate} 
\usepackage{enumitem}
\usepackage{mathrsfs}

\emergencystretch=\maxdimen
\newcommand{\PropAlg}{MiDEX}
\newcommand{\alg}{\mathfrak{A}}
\newcommand{\bE}{\mathbb{E}}
\newcommand{\bP}{\mathbb{P}}
\newcommand{\regret}{R}
\newcommand{\SubsetAt}{\mathcal{A}}
\newcommand{\HistoryAt}{\mathcal{H}}
\newcommand{\WinBattle}{W}
\newcommand{\ind}{\mathds{1}}
\newcommand{\pref}{P}
\DeclareMathOperator*{\defined}{\coloneqq}
\DeclareMathOperator*{\argmax}{\arg\!\max}
\DeclareMathOperator*{\transpose}{\mathsf{T}}
\newcommand\given[1][]{\:#1\vert\:}
\newcommand{\winner}{\textsc{index}}
\newcommand{\mPrime}{m'}
\newcommand{\mDoublePrime}{\mPrime'}
\allowdisplaybreaks

\begin{document}

\twocolumn[
\icmltitle{Adversarial Multi-dueling Bandits}



\icmlsetsymbol{equal}{*}

\begin{icmlauthorlist}
\icmlauthor{Pratik Gajane}{yyy}
\end{icmlauthorlist}

\icmlaffiliation{yyy}{Eindhoven University of Technology, Netherlands}
\icmlcorrespondingauthor{Pratik Gajane}{pratik.gajane@gmail.com}

\icmlkeywords{Machine Learning, ICML}

\vskip 0.3in
]



\printAffiliationsAndNotice{}  

\begin{abstract}
We introduce the problem of regret minimization in adversarial multi-dueling bandits. While adversarial preferences have been studied in dueling bandits, they have not been explored in multi-dueling bandits. In this setting, the learner is required to select $m \geq 2$ arms at each round and observes as feedback the identity of the most preferred arm which is based on an arbitrary preference matrix chosen obliviously. We introduce a novel algorithm, \PropAlg\,(\underline{M}ult\underline{i} \underline{D}ueling \underline{EX}P3), to learn from such preference feedback that is assumed to be generated from a pairwise-subset choice model. We prove that the expected cumulative $T$-round regret of \PropAlg\ compared to a Borda-winner from a set of $K$ arms is upper bounded by $O((K \log K)^{1/3} \, T^{2/3})$. Moreover, we prove a lower bound of $\Omega(K^{1/3} \, T^{2/3})$ for 
the expected regret in this setting which demonstrates that our proposed algorithm is near-optimal. 
\end{abstract}

\section{Introduction}
Multi-armed bandits (MAB) is a sequential decision making framework that involves selecting from multiple options (symbolized as \textit{arms}) with unknown outcomes to optimize performance over time. This framework can be useful in impactful applications like e-healthcare, clinical trials, recommendation systems, and online advertising. 

In a classical MAB problem, the learner selects an arm in each round and observes absolute feedback i.e., a numerical value as feedback for the selected arm. However, in some tasks, especially those requiring human feedback, it is often more practical to elicit preference feedback than absolute feedback. Motivated by such scenarios, there has been a growing body of work on \textit{dueling bandits} in which the learner selects a pair of arms to be compared in each round, and receives preference feedback about the selected pair. Recently, a few works have extended this setup to \textit{multi-dueling bandits} in which the learner selects a subset of $m \geq 2$ arms in each round, and receives preference feedback about the selected arms  
\cite{10.1145/2983323.2983659, saha2018battle, pmlr-v98-saha19a, 10.5555/3454287.3455185, NEURIPS2020_d5fcc35c, DBLP:journals/corr/SuiZBY17, haddenhorst2021identification, 10.5555/3398761.3398806}.

Preferences, either over a pair of arms or for $m \geq 2$ arms, can be expressed as  stochastic \textit{stationary} preferences or \textit{adversarial} preferences.  Stochastic stationary preferences represent scenarios where preferences are assumed to be generated through stochastic models that do not change over time. Such preferences might be unable to capture real-world applications where preferences might vary significantly and unpredictably over time. These preferences would find more faithful representation within a robust worst-case (\textit{adversarial}) model, which avoids the stringent stochastic assumption and allows for an arbitrary sequence of preferences over time. For dueling bandits, several algorithms have been proposed for stochastic stationary preferences \cite{10.1145/1553374.1553527, 10.5555/3104482.3104513, 10.1016/j.jcss.2011.12.028, pmlr-v28-urvoy13, pmlr-v32-zoghi14, pmlr-v40-Komiyama15} and for adversarial preferences \cite{rex3,pmlr-v139-saha21a}. However, to the best of our knowledge, all the previous work on multi-dueling bandits assumes stochastic stationary preferences, and adversarial preferences have not been studied in this context. 

\subsection*{Our Contributions}
\begin{itemize}
    \item We introduce and formalize the problem of regret minimization in adversarial multi-dueling bandits, where the learner is required to select $m \geq 2$ arms at each round and observes as feedback the identity of the most preferred arm. In this general adversarial model, the sequence of preference matrices is allowed to be entirely arbitrary and they are chosen obliviously by the environment. 
    \item We propose a novel algorithm called, \PropAlg\,
    , considering a \textit{pairwise-subset choice model} for feedback (exact definitions will follow in Section \ref{sec:Setting}). 
    \item We analyze the expected cumulative regret of \PropAlg \ compared to a \textit{Borda}-winner (which, unlike the alternative of \textit{Condorcet}-winner, always exists and may suit the adversarial model better).
    Our analysis demonstrates that the expected cumulative regret of \PropAlg \ is upper bounded by $ O((K \log K)^{1/3} \, T^{2/3})$. 
    \item Furthermore, we establish a lower bound of $\Omega(K^{1/3} \, T^{2/3})$ for the expected cumulative regret, indicating the near-optimality of our proposed algorithm.
\end{itemize}

\section{Related Work}
\label{sec:Rel}
In the multi-dueling bandits problem considered in \citet{10.1145/2983323.2983659, 10.1145/2835776.2835804, DBLP:journals/corr/SuiZBY17, 10.5555/3398761.3398806}, the learner is assumed to receive some subset of the possible $ \binom{m}{2}$ pairwise comparisons amongst the selected $m$ arms. In contrast, \citet{saha2018battle, NEURIPS2020_d5fcc35c} assume a more limited form of feedback, referred to as \textit{winner feedback}, where the learner receives only the identity of the arm that is most preferred among the selected arms. In this article, we consider winner feedback. 

In multi-dueling bandits (and dueling bandits), several notions of an optimal arm have been considered in the literature. Many works on multi-dueling bandits use the notion of \textit{Condorcet winner}: an arm being preferred when compared to any other arm.  For instance, \citet{saha2018battle, 10.1145/2983323.2983659, 10.5555/3398761.3398806} consider regret minimization in multi-dueling bandits for stochastic preferences with Condorcet winner. \citet{NEURIPS2020_d5fcc35c} extend this notion to a generalized Condorcet winner:  an arm that has the greatest probability of being the winner in each subset containing it and propose an algorithm for regret minimization. \citet{haddenhorst2021identification} also use the notion of a generalized Condorcet winner and propose an algorithm for best arm identification with bounds on its sample complexity. \citet{pmlr-v98-saha19a} study the problem of identifying a near-best arm with high confidence where the Condorcet winner is considered to be the best arm. All of these works in the framework of multi-dueling bandits assume that the underlying preferences are of a stationary stochastic nature.

As highlighted by \citet{pmlr-v38-jamieson15}, using the notion of a Condorcet winner may pose several drawbacks. Chief among these is the potential non-existence of a Condorcet winner, as illustrated by the absence of one in widely used datasets like MSLR-WEB10k \cite{Qin2010}. Moreover, in the context of adversarial preferences addressed in this study, assuming the presence of a Condorcet winner would imply preferences where a certain fixed arm is consistently preferred to all the other arms at all rounds. Such a constraint might render the framework of adversarial (multi-)dueling bandits that presupposes the existence of a Condorcet winner unsuitable for many real-world applications with non-stationary preferences.

Alternatively, the notion of a \textit{Borda winner} has been used in adversarial dueling bandits \cite{pmlr-v139-saha21a}. A Borda winner is an arm with the highest \textit{Borda score} where the Borda score of an arm is the probability that it is preferred over another arm chosen uniformly at random. Firstly, the advantage of using the notion of Borda winner is that it always exists, unlike a Condorcet winner. Secondly, as argued in \citet{pmlr-v38-jamieson15}, in certain cases a Borda winner represents a better reflection of preferences than a Condorcet winner when they are distinct, and the former is more robust to estimation errors in preferences. Consequently, in this article, we use the notion of a Borda winner.


Other notions of an optimal arm have also been considered for dueling bandits with stochastic preferences: \textit{Copeland winner} \cite{NIPS2015_9872ed9f, pmlr-v48-komiyama16, 10.5555/3157096.3157169} and \textit{von Neumann} winner \cite{pmlr-v49-balsubramani16, pmlr-v40-Dudik15}.

Another tangentially related problem is the one considered in \citet{10.5555/3454287.3455185} where there exists a unique unknown ranking $r_1 \succ r_2, \dots \succ r_K$ such that $i \succ j$ indicates that $i$ is more preferred than $j$; the learner receives winner feedback for the selected $m \geq 2$ arms; and the learner's goal is to recover this true ranking. 

Multi-armed bandits with preference feedback can also be formulated as partial monitoring games which is a rich framework for sequential decision making under uncertainty \cite{GajaneU15, JMLR:v24:22-1248}.    

\section{Problem Setting}
\label{sec:Setting}
We consider an online decision making task over a finite set of arms $[K] \defined \{1, 2, \dots, K\}$ which spans $T$ rounds \footnote{Throughout the article, we use the shorthand of $[V]$ to represent $\{1,2,3\dots,V\}$ for any positive integer $V$.}. At each round $t=1,2, \dots, T$,
\begin{itemize}
    \item the learner selects, possibly at random, a multiset of arms $\SubsetAt_t$ such that $|\SubsetAt_t| = m$ where $2 \leq m \leq K$; and
    \item the learner observes a `winner': an arm that is preferred over all the other arms in $\SubsetAt_t$ at time $t$.
\end{itemize}

The selection of a winner from a multiset of arms is governed by the underlying \textit{subset choice model}. Given a multiset of arms, a subset choice model determines the probability of one of the arms being preferred over the rest in the multiset. In this article, we consider the \textit{pairwise-subset choice model}, introduced by \citet{saha2018battle}. There also exist other subset choice models in the related literature such as a popular class of models called \textit{Random Utility Models} \cite{10.5555/2999134.2999149}.


\subsection{Pairwise-subset Choice Model}
We assume that the environment obliviously fixes a sequence of $T$ preference matrices $\pref_1, \pref_2, \dots, \pref_T$ where each $\pref_t(i,j)$ is the probability that arm $i$ is preferred when compared to arm $j$ at round $t \in [T]$. Each $\pref_t \in [0,1]^{K \times K}$ satisfies $\pref_t(i,j) = 1 - \pref_t(j,i)$ and $\pref_t(i,i) = 1/2$ for all $i,j \in [K]$. These preference matrices are not revealed to the learner. 

Given a multiset of arms $\SubsetAt = \{\SubsetAt(1), \SubsetAt(2), \dots, \SubsetAt(m)\}$ and a corresponding preference matrix $\pref$, the probability of any index $i \in [m]$ being selected as the winner is defined as
\begin{equation*}
\WinBattle(i \given \SubsetAt, \pref)  \defined  \sum_{j=1, j \neq i}^{m} \frac{2P\left(\SubsetAt(i),\SubsetAt(j)\right)}{m(m-1)}.   
\end{equation*}

As noted in \citet{saha2018battle}, the above forms a valid probability distribution over the indices $i \in [m]$, and when $m=2$ (which corresponds to the dueling bandits case), it simplifies to the probability of an arm winning the pairwise duel.

\subsection{Performance Measure: Regret}
The performance of the learner's arm selection strategy is measured against the performance of an optimal arm in hindsight. As noted earlier in Section \ref{sec:Rel}, we use the notion of a Borda winner which is defined using the Borda score defined below.

\begin{definition}[Borda Score]
\label{def:BordaScore}
    The Borda score of an arm $i \in [K]$ according to a preference matrix $\pref_t$ is defied as
    \begin{equation*}
        b_t(i) \defined \frac{1}{K-1}
                \sum_{j \in [K] \setminus \{i\}}
                \pref_t(i, j). 
    \end{equation*} 
\end{definition}
Accordingly, the optimal arm $i^*$ is defined as the arm with the highest cumulative Borda score up to horizon $T$ i.e.,
\begin{equation*}
i^* \defined \argmax_{i \in [K]} \sum_{t=1}^{T}b_t(i).
\end{equation*}

\begin{definition}[Regret]
    \label{def:regret}
    Let $\SubsetAt_t$ be the subset of arms selected by an algorithm at $t=1, \dots, T$ such that $|\SubsetAt_t| = m$. Then regret of the algorithm at the end of horizon $T$ is defined as 
    \begin{equation*}
        \regret_T \defined \sum_{t=1}^{T} \left[ b_t(i^*) - \frac{1}{m} \sum_{i \in \SubsetAt_t} b_t(i) \right].
    \end{equation*}
\end{definition}

In our proposed algorithm, we make use of the \textit{Shifted Borda Score} \cite{pmlr-v139-saha21a}.
\begin{definition}[Shifted Borda Score]
    \label{def:SBordaScore}
    The shifted Borda score of an arm $i \in [K]$ according to a preference matrix $\pref_t$ is defined as 
    \begin{equation*}
        s_t(i) \defined \frac{1}{K}
                \sum_{j \in [K]}
                \pref_t(i, j). 
    \end{equation*}
\end{definition}

\begin{definition}[Shifted Borda Regret]
     Let $\SubsetAt_t$ be the subset of arms selected by an algorithm at $t=1, \dots, T$ such that $|\SubsetAt_t| = m$. Then shifted Borda regret of the algorithm at the end of horizon $T$ is defined as 
     $$
   \regret^s_T \defined \sum_{t=1}^{T} \left[ s_t(i^*) -  \frac{1}{m} \sum_{i \in \SubsetAt_t} s_t(i) \right].
     $$
\end{definition}

The following proposition lets us interpret the shifted Borda score of an arm in terms of its Borda score. 
\begin{proposition}
\label{prop1}
   The shifted Borda score $s_t(i)$ of any arm $i \in [K]$ is related to its Borda score $b_t(i)$ by the equation
    $$
    s_t(i) = \frac{K-1}{K} \, b_t(i) + \frac{1}{2K}.
    $$ 
\end{proposition}
\begin{proof}
    \begin{align*}
         s_t(i) &= \frac{1}{K}
                \sum_{j \in [K]}
                \pref_t(i, j) \\
                &= \frac{1}{K}
                 \sum_{j \in [K] \setminus \{i\}} \, \pref_t(i,j) + \frac{1}{K} \,\pref_t(i,i) \\
                &= \frac{K-1}{K} \, b_t(i) + \frac{1}{2K},
    \end{align*}
    where the last equality follows from Definition \ref{def:BordaScore} and the fact that $\pref_t(i,i) = \frac{1}{2}$ for any $i \in [K]$.
\end{proof}
Using the above, we can state the following for optimal arm $i^*$ and regret $\regret_T$. 
\begin{proposition}
\label{prop2}
$ i^* \defined \argmax_{i \in [K]} \sum_{t=1}^{T}b_t(i) = \argmax_{i \in [K]} \sum_{t=1}^{T}s_t(i).$
\end{proposition}

\begin{proposition} 
\label{prop3}
$\regret_T = \frac{K}{K-1} \regret^s_T$.
\end{proposition}

\section{Our Algorithm and Performance Guarantee}
In this Section, we provide our proposed algorithm \PropAlg \,(\underline{M}ult\underline{i} \underline{D}ueling \underline{EX}P3). It falls under the class of \textit{Exponential Weight} algorithms --- a well-known class of algorithms for MAB problems that can be traced back to \citet{EXP3}. 

In \PropAlg, firstly at each round $t=1,2,\dots, T$, two arms $x_t$ and $y_t$ are sampled from $q_t$. Then each of $x_t$ and $y_t$ is replicated about $\frac{m}{2}$ times to constitute the multiset of arms $\SubsetAt_t$ to be selected at time $t$. After receiving the winner index from $\SubsetAt_t$ according to the pairwise-subset choice model as defined in Eq. \eqref{eq:WinningIndex},  Procedure \ref{Proc:G} transforms the received feedback which is then used to compute an estimate of the shifted Borda score $\hat{s}_t(i)$ for each arm $i$. These estimates are, in turn, used to compute $q_{t+1}$. A parameter $\gamma \in (0,1]$ is incorporated to ensure that for all $t \in [T]$ and $i \in [K]$, $q_{t}(i) \geq \gamma/K$ which translates to the selection probability of any arm always being above zero.  

\begin{algorithm}[!t]
   \caption{\PropAlg \,(\underline{M}ult\underline{i} \underline{D}ueling \underline{EX}P3)}
   \label{alg:example}
    \begin{algorithmic}[1]
    \STATE {\bfseries Input:} Set of arms $[K]$,  horizon $T$, number of arms to be selected at each round $m$, exploration parameter $\gamma \in (0,1]$ and learning rate $\eta>0$.
    \STATE \textbf{Initialize:} Initial arm-selection probability distribution $q_1(i) = 1/K$, $\forall i \in [K]$.
    \FOR{$t=1, 2, \dots, T$}
    \STATE Sample $x_t, y_t$ $ \sim q_t$ i.i.d. with replacement.
    \STATE Construct $\SubsetAt_t = \{x_t, x_t,\dots, x_t, y_t, y_t, \dots, y_t\}$ by replicating $x_t$  for $\lceil \frac{m}{2} \rceil$ times and $y_t$ for $\lfloor \frac{m}{2} \rfloor$ times with probability $\frac{1}{2}$, or $x_t$ for $\lfloor \frac{m}{2} \rfloor$ times and $y_t$ for $\lceil \frac{m}{2} \rceil$ time with probability $\frac{1}{2}$.
    \STATE Receive winning index 
    \begin{equation}
    \label{eq:WinningIndex}
            \winner_t \sim \WinBattle_t(i|\SubsetAt_t)        
    \end{equation}
    where $W_t(i|\SubsetAt_t ) = \sum_{j=1, j \neq i}^{m} \frac{2\pref_t\left(\SubsetAt_t(i),\SubsetAt_t(j)\right)}{m(m-1)}$ for any $i \in [m]$ and $\SubsetAt_t(i)$ is the $i^{th}$ item in $\SubsetAt_t$. 
    \STATE  \begin{tabbing}
    \hspace{0.8em} \= \hspace{0.8em} \= \kill
    \textbf{if} \> $\SubsetAt_t(\winner_t) = x_t$ \textbf{then}  \\
    \> $o_t = x_t$, \\
    \textbf{else} \\
    \> $o_t = y_t$.\\
    \textbf{end if}
\end{tabbing}    
    \STATE Estimates scores, for all $i \in [K]$:
            \begin{equation}
                \label{def:SBordaEstimate}
                \hat{s}_t(i) \defined \frac{\ind(i=x_t)}{K \, q_t(i)} \sum_{j \in [K]} \frac{\ind(j=y_t) \ g(m,o_t,x_t) }{q_t(j)},
            \end{equation}
            where $g(m,o_t,x_t)$ is computed as shown in Procedure \ref{Proc:G}. 
    \STATE Update, for all $i \in [K]$:
            \begin{align}
                \tilde{q}_{t+1}(i) &\defined \frac{\exp{\left(\eta \sum_{\tau=1}^{t} \hat{s}_\tau(i) \right)}}{\sum_{j=1}^{K} \exp{\left( \eta \sum_{\tau=1}^{t} \hat{s}_\tau(j) \right)}};\nonumber \\
                {q}_{t+1}(i)& \defined (1 - \gamma) \, \tilde{q}_{t+1}(i) \, + \, \frac{\gamma}{K}.
                \label{def:WeightUpdate}
            \end{align}
   \ENDFOR
\end{algorithmic}
\end{algorithm}

\begingroup
\captionsetup[algorithm]{name=Procedure}
\begin{algorithm}
\caption{$g(m,o_t,x_t)$}
\label{Proc:G}
 \begin{algorithmic} 
\STATE \begin{tabbing}
    \hspace{0.8em} \= \hspace{0.8em} \= \kill
    \textbf{if} \> $m$ is even \textbf{then} \\
    \> \textbf{return} $\frac{\ind(o_t = x_t) - \frac{(m-2)}{4(m-1)}}{\frac{m}{2(m-1)}}$, \\
    \textbf{else} \\
    \> \textbf{return} $\frac{\ind(o_t = x_t) - \frac{(m-1)}{4m}}{\frac{m+1}{2m}}$. \\
    \textbf{endif}
\end{tabbing}
\end{algorithmic}
\end{algorithm}
\endgroup

\begin{theorem}
\label{thm:main}
Let $\gamma = \sqrt{\frac{3\eta K}{2}}$ and $\eta = \left( \frac{2 \log K}{T \sqrt{K} \, \mPrime} \right)^{2/3}$ where $\mPrime = \left(\sqrt{\frac{3}{2}} \, + \, \sqrt{\frac{2}{3}} \, \frac{(3m+1)^2}{4(m+1)^2} \right)$. For any $T$, $K \geq 2$ and $m \geq 2$, the expected regret of \PropAlg \  satisfies 
\begin{equation*}
    \bE[\regret_T] \leq 3.78 \, (\mPrime)^{2/3} \, (K \log K)^{1/3} \, T^{2/3}.
\end{equation*}
\end{theorem}
The regret upper bound can be further simplified to 
$$
 \bE[\regret_T] \leq 8.13 \, (K \log K)^{1/3} \, T^{2/3},
$$
for any $m \geq 2$.

\section{Mathematical Analysis}
The proof of Theorem \ref{thm:main} builds upon the following lemmas. The most important lemmas are Lemma \ref{lem:ExpG} and Lemma \ref{lem:ExpScore}. Lemma \ref{lem:ExpG} proves how the transformed feedback can be interpreted as the probability of $x_t$ winning the duel against $y_t$. Lemma \ref{lem:ExpScore} proves that the score $\hat{s}_t(i)$ being computed in Eq. \eqref{def:SBordaEstimate} is an unbiased estimate of the true shifted Borda score $s_t(i)$. Proofs for the following lemmas can be found in the Appendix. 

\begin{restatable}{lemma}{LemmaExpG}
    \label{lem:ExpG}
    $\bE[g(m,o_t,x_t)] = \pref_t(x_t, y_t)$.
\end{restatable}
Lemma \ref{lem:ExpG} is proved using Procedure \ref{Proc:G}, the construction of $\SubsetAt_t$ and the definition of $\WinBattle_t(i \given \SubsetAt_t)$.

\begin{restatable}{lemma}{LemmaExpScore}\label{lem:ExpScore}
    For all $t \in [T]$ and $i \in [K]$, it holds that $\bE[\hat{s}_t(i)] = s_t(i)$.  
\end{restatable}
Lemma \ref{lem:ExpScore} is proved using Lemma \ref{lem:ExpG} and the fact that $x_t$ and $y_t$ are sampled i.i.d. from $q_t$ with replacement.
 
Next, in Lemma \ref{lem:GUpperBound}, we bound the magnitude of the transformed feedback $g(m,o_t,x_t)$. 

\begin{restatable}{lemma}{LemmaGUpperBound}
\label{lem:GUpperBound}
    For all $t \in [T]$ and $m\geq 2$, $g(m,o_t,x_t) \leq \frac{3m+1}{2m+2}$.
\end{restatable}
Lemma \ref{lem:GUpperBound} follows from expanding the construction of $g(m,o_t,x_t)$ given in Procedure \ref{Proc:G}.

In Lemma \ref{lem:ScoreUpperBound}, we bound the magnitude of the shifted Borda score estimates. 
\begin{restatable}{lemma}{LemmaScoreUpperBound}\label{lem:ScoreUpperBound}
    Let $\gamma \geq  \sqrt{3\eta K/2}$. Then, for any $t \in [T]$, $i \in [K]$ and $\eta>0$, it holds that $\eta \hat{s}_t(i) \in [0,1]$.
\end{restatable}
Lemma \ref{lem:ScoreUpperBound} is proved using Lemma \ref{lem:GUpperBound},  the definition of $q_t$ given in Eq. \eqref{def:WeightUpdate} and the definition of $\hat{s}_t$ given in Eq. \eqref{def:SBordaEstimate}.

Let $\HistoryAt_{t-1} \defined (q_1, \pref_1, x_1, y_1, o_1, \dots, q_t, \pref_t)$ denote the history up to round $t$. 
\begin{restatable}{lemma}{LemmaExpWeightTScore}
    \label{lem:ExpWeightTScore}
    For all $t \in [T]$, it holds that $\bE_{\HistoryAt_t} \left[q_t^{\transpose} \ \hat{s}_t \right] = \bE_{\HistoryAt_{t-1}} \bigl[ \bE_{i \sim q_t} \left[ s_t(i) \given \HistoryAt_{t-1} \right] \bigr]$.
\end{restatable}

Lemma \ref{lem:ExpWeightTScore} follows from the proof of Lemma \ref{lem:ExpScore}. 
\begin{restatable}{lemma}{LemmaExpWeightTScoreSquared}
    \label{lem:ExpWeightTScoreSquared}
    At any time $t \in [T]$, it holds that $\bE\left[ \sum_{i=1}^{K} q_t(i) \ \hat{s}_t(i)^2 \right] \leq  \frac{(3m+1)^2}{4(m+1)^2} \, \frac{K}{\gamma}$.
\end{restatable}
Lemma \ref{lem:ExpWeightTScoreSquared} is proved using Lemma \ref{lem:GUpperBound}, the definition of $\hat{s}_t$ given in Eq. \eqref{def:SBordaEstimate}, and the fact that $\forall i' \in [K]$ and $\forall t \in [T]$, $q_t(i') \geq \gamma/K$ according to Eq. \eqref{def:WeightUpdate}, the initialization of $q_t$, and $\gamma \in (0,1]$.

\begin{restatable}{lemma}{LemmaSampling}
    \label{lem:Sampling}
    For any $i \in [K]$, $j \in [m]$ and $t \in [T]$,
    $$
    \bP\Bigl(\SubsetAt_t(j) = i\Bigr) = q_t(i).
    $$
\end{restatable}
Lemma \ref{lem:Sampling} follows from the construction of $\SubsetAt_t$ and the fact that $x_t$ and $y_t$ are sampled i.i.d. from $q_t$ with replacement.

\subsection{Proof of Theorem \ref{thm:main}}
\begin{proof}
We start by expanding the expression for the expectation of shifted Borda regret $\regret^s_T$. 
    \begin{align}
        \bE_{\HistoryAt_T}[\regret^s_T] &= \bE_{\HistoryAt_T}\left[ \sum_{t=1}^{T} \left[ s_t(i^*) - \frac{1}{m} \sum_{j \in \SubsetAt_t} s_t(j) \right] \right] \nonumber \\
         &=  \sum_{t=1}^{T} s_t(i^*) - \sum_{t=1}^{T} \bE_{\HistoryAt_t} \left[\frac{1}{m} \sum_{j \in \SubsetAt_t} s_t(j) \right] \nonumber \\
        %
        %
        &= \sum_{t=1}^{T} s_t(i^*) - \sum_{t=1}^{T} \bE_{\HistoryAt_{t-1}} \left[\bE_{i \sim q_t} \left[ s_t(i) \given \HistoryAt_{t-1} \right] \right] \label{eq:SimplifiedRegret} 
    \end{align}
    In the above, the second equality holds because the preference matrices $\pref_t$ are chosen obliviously, and hence $s_t$ and the identity of $i^*$ remain independent of the randomness of the algorithm. Moreover, the last equality uses that all the $m$ arms in $\SubsetAt_t$ are ${\sim} \, q_t$ (Lemma \ref{lem:Sampling}).  

For any $\gamma \geq  \sqrt{3\eta K/2}$ and $\eta > 0$, we have $\eta \hat{s}_t(i) \in [0,1]$ using Lemma \ref{lem:ScoreUpperBound}. Using the regret guarantee of standard Exponential Weight algorithm \citep{EXP3} over the completely observed fixed sequence of reward vectors $\hat{s}_1, \hat{s}_2, \dots, \hat{s}_T$, 
    for $i^* \defined \argmax_{i \in [K]} \sum_{t=1}^{T}b_t(i) =  \argmax_{i \in [K]} \sum_{t=1}^{T}s_t(i)$, one can state that
    \begin{align*}
        \sum_{t=1}^{T} \hat{s}_t(i^*) - \sum_{t=1}^{T} \tilde{q}^{\transpose}_t \hat{s}_t 
        \leq 
        \frac{\log K}{\eta} + \eta \sum_{t=1}^{T} \sum_{i=1}^{K} \tilde{q}_t(i) \hat{s}_t(i)^2.
    \end{align*}

Using $\tilde{q}_t = \frac{q_t - \frac{\gamma}{K}}{1 - \gamma}$ and $\gamma \in (0,1)$, with the above inequality, we have that
\begin{align*}
&   (1-\gamma)\sum_{t=1}^{T} \hat{s}_t(i^*) 
    - \sum_{t=1}^{T} {q}^{\transpose}_t \hat{s}_t \\
& \qquad \ \leq \
    \frac{\log K}{\eta} + 
    \eta \sum_{t=1}^{T} \sum_{i=1}^{K} {q}_t(i) \hat{s}_t(i)^2 \\
\implies & (1-\gamma) \sum_{t=1}^{T} \bE_{\HistoryAt_T} \left[\hat{s}_t(i^*)\right] 
    - \sum_{t=1}^{T} \bE_{\HistoryAt_T} \left[{q}^{\transpose}_t \hat{s}_t \right] \\
& \qquad \leq \
    \frac{\log K}{\eta} + 
    \eta \sum_{t=1}^{T} \bE_{\HistoryAt_T} \left[\sum_{i=1}^{K} \left[{q}_t(i) \hat{s}_t(i)^2 \right] \right] \\
\overset{(a)}{\implies }&
    (1-\gamma) \sum_{t=1}^{T} s_t(i^*) - 
    \sum_{t=1}^{T} \bE_{\HistoryAt_{t-1}} \bigl[ \bE_{i \sim q_t} \left[ s_t(i) \given \HistoryAt_{t-1} \right] \bigr] \\
& \qquad \leq \
    \frac{\log K}{\eta} + \eta \sum_{t=1}^{T} \frac{(3m+1)^2}{4(m+1)^2} \, \frac{K}{\gamma} \\
\implies &
    \sum_{t=1}^{T} s_t(i^*) - 
    \sum_{t=1}^{T} \bE_{\HistoryAt_{t-1}} \bigl[ \bE_{i \sim q_t} \left[ s_t(i) \given \HistoryAt_{t-1} \right] \bigr] \\
& \qquad \leq \ \gamma \sum_{t=1}^{T} s_t(i^*) + 
    \frac{\log K}{\eta} + \frac{(3m+1)^2}{4(m+1)^2} \, \frac{\eta KT}{\gamma} \\
\overset{(b)}{\implies} &
    \bE_{\HistoryAt_T}[\regret^s_T]  \\
& \qquad \leq \ \gamma T + 
    \frac{\log K}{\eta} + \frac{(3m+1)^2}{4(m+1)^2} \, \frac{\eta KT}{\gamma} \\ 
\overset{(c)}{\implies} &
    \bE_{\HistoryAt_T}[\regret^s_T]  \\
& \qquad \leq \ \sqrt{ \frac{3\eta K}{2}} \, T + 
    \frac{\log K}{\eta} + \frac{(3m+1)^2}{4(m+1)^2} \, \sqrt{ \frac{2\eta K}{3}} \, T \\ 
\overset{(d)}{\implies} &
    \bE_{\HistoryAt_T}[\regret^s_T]  \\
& \qquad \leq \ 1.89 \, (\mPrime)^{2/3} \, (K \log K)^{1/3} \, T^{2/3},
\end{align*}
where $\mPrime = \left(\sqrt{\frac{3}{2}} \, + \, \sqrt{\frac{2}{3}} \frac{(3m+1)^2}{4(m+1)^2} \right)$.
In the above, $(a)$ follows from Lemma \ref{lem:ExpScore}, Lemma \ref{lem:ExpWeightTScore} and Lemma \ref{lem:ExpWeightTScoreSquared};  $(b)$ follows from Eq. \eqref{eq:SimplifiedRegret} and using $s_t(i^*) \leq 1$; (c)
follows from setting $\gamma = \sqrt{3\eta K/2}$; and $(d)$ follows from optimizing over $\eta$ which gives $\eta = \left( \frac{2 \log K}{T \sqrt{K} \, \mPrime} \right)^{2/3}$.

The theorem follows by using $\regret_T = \frac{K}{K-1}\regret^s_T$ for any $K \geq 2$ and $T>0$.  
\end{proof}    

\subsection{Varying $m_t$}
Note that \PropAlg\ is also applicable when the number of arms to be selected is time-dependent. In this setting, at each round $t$, the learner receives an integer $m_t|_{2 \leq m_t < K}$ which indicates the number of arms to be selected at time $t$. \PropAlg \  can be employed here with $m$ being replaced with $m_t$ and the corresponding regret bound would feature $\mDoublePrime = \max_{m \in \{m_1, m_2,\dots, m_T\}} \left(\sqrt{\frac{3}{2}} \, + \, \sqrt{\frac{2}{3}} \, \frac{(3m+1)^2}{4(m+1)^2} \right)$ instead of $\mPrime$. The proof structure remains the same with the upper bound in Lemma \ref{lem:GUpperBound} being updated to $ \max_{m \in \{m_1, m_2,\dots, m_T\}} \frac{(3m+1)}{2(m+1)}$. The subsequent proofs and computations build upon this updated bound to arrive at the regret upper bound featuring $\mDoublePrime$.

\section{Lower Bound}
To prove the lower bound for adversarial multi-dueling bandits, we use a reduction from adversarial dueling bandits to adversarial multi-dueling bandits given in Algorithm \ref{alg:reduction}.  
That is we show how an algorithm $\alg_{MB}$ designed for adversarial multi-dueling bandits can be used to solve an instance of adversarial dueling bandits $DB$.  

\begin{algorithm}
\caption{ $\alg_{DB}$: Reduction from adversarial dueling bandits to adversarial multi-dueling bandits}
\label{alg:reduction}
    \begin{algorithmic}[1]
        \FOR{t=1,2,\dots}
            \STATE $\SubsetAt_t = \{\SubsetAt_t(1), \SubsetAt_t(2), \dots, \SubsetAt_t(m)\} \leftarrow$ multiset of arms played by $\alg_{MD}$ at round $t$.
            \STATE Sample $i_t, j_t$  from $[m]$ uniformly at random without replacement. 
            \STATE Play $\Bigl(\SubsetAt_t(i_t), \SubsetAt_t(j_t) \Bigr)$ where $\SubsetAt_t(i)$ is the $i^{th}$ item in $\SubsetAt_t$.
            \STATE Receive $w_t \sim \textsc{Bernoulli}\biggl(\pref_t \Bigl (\SubsetAt_t(i_t), \SubsetAt_t(j_t) \Bigr)
            \biggr).$
            \STATE Return $\winner_t = i_t w_t + j_t(1 - w_t) \in \{i_t, j_t\}$ as the winning index to $\alg_{MB}$.
        \ENDFOR
    \end{algorithmic}
\end{algorithm}

Note that even though this reduction is the same as the reduction suggested by \citet{saha2018battle} for stochastic multi-dueling bandits, our novel contribution is the lemma below which shows that Algorithm \ref{alg:reduction} preserves the expected regret for any arbitrary sequence of preference matrices.

\begin{restatable}{lemma}{LemLowerBound}
\label{lem:LowerBound}
Using $\alg_{DB}$ given in Algorithm \ref{alg:reduction}, $$\bE[\regret_T(\alg_{DB})] = \regret_T(\alg_{MB}),$$
for any arbitrary sequence of preference matrices $\pref_1, \pref_2, \dots, \pref_T$.
\end{restatable}
The complete proof can be found in the Appendix. Here we provide a brief outline of the proof. \\
\textit{Proof Outline}.\\
Let $r_t(\alg_{DB}) \defined b_t(i^*) - \frac{b_t \bigl(\SubsetAt_t(i_t)\bigr) + b_t\bigl(\SubsetAt_t(j_t)\bigr)}{2}$ be the instantaneous regret of $\alg_{DB}$ at round $t$. Correspondingly, let $r_t(\alg_{MB}) \defined  b_t(i^*) - \frac{1}{m} \left[ \sum_{i=1}^{m} b_t \Bigl(\SubsetAt_t(i)\Bigr)\right]$ be the instantaneous regret of $\alg_{MB}$ at round $t$. Firstly, we show that 
$\bE_{i_t, j_t  \overset{Unif}{\sim} [m], i_t \neq j_t} [r_t (\alg_{DB})] = r_t(\alg_{MB})$. 
Then,
\begin{align*}
    \bE[\regret_T(\alg_{DB})] = \sum_{t=1}^{T} \bE[r_t (\alg_{DB})] &=  \sum_{t=1}^{T} r_t(\alg_{MB}) \\
    &= \regret_T(\alg_{MB}).
\end{align*}
\qed

Using the above reduction and Lemma \ref{lem:LowerBound}, along with the lower bound proved for adversarial dueling bandits \cite{pmlr-v139-saha21a}[Theorem 16], we can state the following lower bound for the expected regret of adversarial multi-dueling bandits measured against a Borda winner. 
\begin{theorem}
    \label{thm:LowerBound}
    For any learning algorithm $\alg$, there exists an instance of adversarial multi-dueling bandits with $T\geq K$, $K \geq 4$ and a sequence of preferences $\pref_1, \pref_2, \dots, \pref_T$, such that the expected regret of $\alg$ for that instance is at least $\Omega( K^{1/3} \, T^{2/3})$. 
\end{theorem}


\section{Concluding Remarks}
In conclusion, we have introduced and formalized the problem of regret minimization in adversarial multi-dueling bandits, extending previous research on multi-armed bandits with preference feedback. Our work addresses a gap in the literature by considering scenarios where the learner selects multiple arms at each round and observes the identity of the most preferred arm, based on arbitrary preference matrices. Central to our contribution is the development of a novel algorithm, \PropAlg, tailored to learn from preference feedback following a pairwise-subset choice model. Through rigorous analysis, we have demonstrated that \PropAlg \ achieves near-optimal performance in terms of its expected cumulative regret measured against a Borda winner. Specifically, our upper bound on the expected cumulative regret of \PropAlg \ is of the order $O((K \log K)^{1/3} \, T^{2/3})$. We also prove a matching lower bound of $\Omega(K^{1/3} \, T^{2/3})$, thereby demonstrating the near-optimality of our proposed algorithm up to a logarithmic factor. Future research directions include conducting high-probability regret analysis and exploring the dynamic regret objective with respect to a time-varying benchmark. Another valuable direction would be to investigate alternative notions for optimal arm and subset choice models.
It would also be advantageous to develop a meta-algorithm for multi-dueling bandits which can make use of the corresponding algorithm for dueling bandits as a black-box leading us to incorporate the advancements in dueling bandits into multi-dueling bandits as done for other problems (e.g., \citet{ijcai2023p413}). 

\section*{Acknowledgements}
This work is supported by the Dutch Research Council (NWO) in the framework of EDIC project (grant number 628.011.021). 

\section*{Impact Statement}
This article presents work whose goal is to advance the field of reinforcement learning theory. There are many potential societal consequences of our work, none of which we feel must be specifically highlighted here.

\bibliography{main}

\begin{thebibliography}{29}
\providecommand{\natexlab}[1]{#1}
\providecommand{\url}[1]{\texttt{#1}}
\expandafter\ifx\csname urlstyle\endcsname\relax
  \providecommand{\doi}[1]{doi: #1}\else
  \providecommand{\doi}{doi: \begingroup \urlstyle{rm}\Url}\fi

\bibitem[Agarwal et~al.(2020)Agarwal, Johnson, and Agarwal]{NEURIPS2020_d5fcc35c}
Agarwal, A., Johnson, N., and Agarwal, S.
\newblock Choice bandits.
\newblock In Larochelle, H., Ranzato, M., Hadsell, R., Balcan, M., and Lin, H. (eds.), \emph{Advances in Neural Information Processing Systems}, volume~33, pp.\  18399--18410. Curran Associates, Inc., 2020.
\newblock URL \url{https://proceedings.neurips.cc/paper_files/paper/2020/file/d5fcc35c94879a4afad61cacca56192c-Paper.pdf}.

\bibitem[Auer et~al.(2002)Auer, Cesa-Bianchi, Freund, and Schapire]{EXP3}
Auer, P., Cesa-Bianchi, N., Freund, Y., and Schapire, R.~E.
\newblock The nonstochastic multiarmed bandit problem.
\newblock \emph{SIAM Journal on Computing}, 32\penalty0 (1):\penalty0 48--77, 2002.
\newblock \doi{10.1137/S0097539701398375}.
\newblock URL \url{https://doi.org/10.1137/S0097539701398375}.

\bibitem[Balsubramani et~al.(2016)Balsubramani, Karnin, Schapire, and Zoghi]{pmlr-v49-balsubramani16}
Balsubramani, A., Karnin, Z., Schapire, R.~E., and Zoghi, M.
\newblock Instance-dependent regret bounds for dueling bandits.
\newblock In Feldman, V., Rakhlin, A., and Shamir, O. (eds.), \emph{29th Annual Conference on Learning Theory}, volume~49 of \emph{Proceedings of Machine Learning Research}, pp.\  336--360, Columbia University, New York, New York, USA, 23--26 Jun 2016. PMLR.
\newblock URL \url{https://proceedings.mlr.press/v49/balsubramani16.html}.

\bibitem[Brost et~al.(2016)Brost, Seldin, Cox, and Lioma]{10.1145/2983323.2983659}
Brost, B., Seldin, Y., Cox, I.~J., and Lioma, C.
\newblock Multi-dueling bandits and their application to online ranker evaluation.
\newblock In \emph{Proceedings of the 25th ACM International on Conference on Information and Knowledge Management}, CIKM '16, pp.\  2161–2166, New York, NY, USA, 2016. Association for Computing Machinery.
\newblock ISBN 9781450340731.
\newblock \doi{10.1145/2983323.2983659}.
\newblock URL \url{https://doi.org/10.1145/2983323.2983659}.

\bibitem[Du et~al.(2020)Du, Wang, and Huang]{10.5555/3398761.3398806}
Du, Y., Wang, S., and Huang, L.
\newblock Dueling bandits: From two-dueling to multi-dueling.
\newblock In \emph{Proceedings of the 19th International Conference on Autonomous Agents and MultiAgent Systems}, AAMAS '20, pp.\  348–356, Richland, SC, 2020. International Foundation for Autonomous Agents and Multiagent Systems.
\newblock ISBN 9781450375184.

\bibitem[Dudík et~al.(2015)Dudík, Hofmann, Schapire, Slivkins, and Zoghi]{pmlr-v40-Dudik15}
Dudík, M., Hofmann, K., Schapire, R.~E., Slivkins, A., and Zoghi, M.
\newblock Contextual dueling bandits.
\newblock In \emph{Proceedings of The 28th Conference on Learning Theory}, pp.\  563--587, 2015.

\bibitem[Gajane \& Urvoy(2015)Gajane and Urvoy]{GajaneU15}
Gajane, P. and Urvoy, T.
\newblock Utility-based dueling bandits as a partial monitoring game.
\newblock In \emph{the 12th European Workshop on Reinforcement Learning}, 2015.

\bibitem[Gajane et~al.(2015)Gajane, Urvoy, and Clérot]{rex3}
Gajane, P., Urvoy, T., and Clérot, F.
\newblock A relative exponential weighing algorithm for adversarial utility-based dueling bandits.
\newblock In Bach, F. and Blei, D. (eds.), \emph{Proceedings of the 32nd International Conference on Machine Learning}, volume~37 of \emph{Proceedings of Machine Learning Research}, pp.\  218--227, Lille, France, 07--09 Jul 2015. PMLR.
\newblock URL \url{https://proceedings.mlr.press/v37/gajane15.html}.

\bibitem[Gajane et~al.(2023)Gajane, Auer, and Ortner]{ijcai2023p413}
Gajane, P., Auer, P., and Ortner, R.
\newblock Autonomous exploration for navigating in mdps using blackbox rl algorithms.
\newblock In Elkind, E. (ed.), \emph{Proceedings of the Thirty-Second International Joint Conference on Artificial Intelligence, {IJCAI-23}}, pp.\  3714--3722. International Joint Conferences on Artificial Intelligence Organization, 8 2023.
\newblock \doi{10.24963/ijcai.2023/413}.
\newblock URL \url{https://doi.org/10.24963/ijcai.2023/413}.
\newblock Main Track.

\bibitem[Haddenhorst et~al.(2021)Haddenhorst, Bengs, and H{\"u}llermeier]{haddenhorst2021identification}
Haddenhorst, B., Bengs, V., and H{\"u}llermeier, E.
\newblock Identification of the generalized condorcet winner in multi-dueling bandits.
\newblock In Beygelzimer, A., Dauphin, Y., Liang, P., and Vaughan, J.~W. (eds.), \emph{Advances in Neural Information Processing Systems}, 2021.
\newblock URL \url{https://openreview.net/forum?id=omDF-uQ_OZ}.

\bibitem[Jamieson et~al.(2015)Jamieson, Katariya, Deshpande, and Nowak]{pmlr-v38-jamieson15}
Jamieson, K., Katariya, S., Deshpande, A., and Nowak, R.
\newblock {Sparse Dueling Bandits}.
\newblock In Lebanon, G. and Vishwanathan, S. V.~N. (eds.), \emph{Proceedings of the Eighteenth International Conference on Artificial Intelligence and Statistics}, volume~38 of \emph{Proceedings of Machine Learning Research}, pp.\  416--424, San Diego, California, USA, 09--12 May 2015. PMLR.
\newblock URL \url{https://proceedings.mlr.press/v38/jamieson15.html}.

\bibitem[Kirschner et~al.(2023)Kirschner, Lattimore, and Krause]{JMLR:v24:22-1248}
Kirschner, J., Lattimore, T., and Krause, A.
\newblock Linear partial monitoring for sequential decision making: Algorithms, regret bounds and applications.
\newblock \emph{Journal of Machine Learning Research}, 24\penalty0 (346):\penalty0 1--45, 2023.
\newblock URL \url{http://jmlr.org/papers/v24/22-1248.html}.

\bibitem[Komiyama et~al.(2015)Komiyama, Honda, Kashima, and Nakagawa]{pmlr-v40-Komiyama15}
Komiyama, J., Honda, J., Kashima, H., and Nakagawa, H.
\newblock Regret lower bound and optimal algorithm in dueling bandit problem.
\newblock In Grünwald, P., Hazan, E., and Kale, S. (eds.), \emph{Proceedings of The 28th Conference on Learning Theory}, volume~40 of \emph{Proceedings of Machine Learning Research}, pp.\  1141--1154, Paris, France, 03--06 Jul 2015. PMLR.
\newblock URL \url{https://proceedings.mlr.press/v40/Komiyama15.html}.

\bibitem[Komiyama et~al.(2016)Komiyama, Honda, and Nakagawa]{pmlr-v48-komiyama16}
Komiyama, J., Honda, J., and Nakagawa, H.
\newblock Copeland dueling bandit problem: Regret lower bound, optimal algorithm, and computationally efficient algorithm.
\newblock In Balcan, M.~F. and Weinberger, K.~Q. (eds.), \emph{Proceedings of The 33rd International Conference on Machine Learning}, volume~48 of \emph{Proceedings of Machine Learning Research}, pp.\  1235--1244, New York, New York, USA, 20--22 Jun 2016. PMLR.
\newblock URL \url{https://proceedings.mlr.press/v48/komiyama16.html}.

\bibitem[Qin et~al.(2010)Qin, Liu, Xu, and Li]{Qin2010}
Qin, T., Liu, T.-Y., Xu, J., and Li, H.
\newblock Letor: A benchmark collection for research on learning to rank for information retrieval.
\newblock \emph{Information Retrieval}, 13\penalty0 (4):\penalty0 346--374, Aug 2010.
\newblock ISSN 1573-7659.
\newblock \doi{10.1007/s10791-009-9123-y}.
\newblock URL \url{https://doi.org/10.1007/s10791-009-9123-y}.

\bibitem[Ren et~al.(2019)Ren, Liu, and Shroff]{10.5555/3454287.3455185}
Ren, W., Liu, J., and Shroff, N.~B.
\newblock \emph{On sample complexity upper and lower bounds for exact ranking from noisy comparisons}.
\newblock Curran Associates Inc., Red Hook, NY, USA, 2019.

\bibitem[Saha \& Gopalan(2018)Saha and Gopalan]{saha2018battle}
Saha, A. and Gopalan, A.
\newblock Battle of bandits.
\newblock In Globerson, A. and Silva, R. (eds.), \emph{Proceedings of the Thirty-Fourth Conference on Uncertainty in Artificial Intelligence, {UAI} 2018, Monterey, California, USA, August 6-10, 2018}, pp.\  805--814. {AUAI} Press, 2018.
\newblock URL \url{http://auai.org/uai2018/proceedings/papers/290.pdf}.

\bibitem[Saha \& Gopalan(2019)Saha and Gopalan]{pmlr-v98-saha19a}
Saha, A. and Gopalan, A.
\newblock Pac battling bandits in the plackett-luce model.
\newblock In Garivier, A. and Kale, S. (eds.), \emph{Proceedings of the 30th International Conference on Algorithmic Learning Theory}, volume~98 of \emph{Proceedings of Machine Learning Research}, pp.\  700--737. PMLR, 22--24 Mar 2019.
\newblock URL \url{https://proceedings.mlr.press/v98/saha19a.html}.

\bibitem[Saha et~al.(2021)Saha, Koren, and Mansour]{pmlr-v139-saha21a}
Saha, A., Koren, T., and Mansour, Y.
\newblock Adversarial dueling bandits.
\newblock In Meila, M. and Zhang, T. (eds.), \emph{Proceedings of the 38th International Conference on Machine Learning}, volume 139 of \emph{Proceedings of Machine Learning Research}, pp.\  9235--9244. PMLR, 18--24 Jul 2021.
\newblock URL \url{https://proceedings.mlr.press/v139/saha21a.html}.

\bibitem[Schuth et~al.(2016)Schuth, Oosterhuis, Whiteson, and de~Rijke]{10.1145/2835776.2835804}
Schuth, A., Oosterhuis, H., Whiteson, S., and de~Rijke, M.
\newblock Multileave gradient descent for fast online learning to rank.
\newblock In \emph{Proceedings of the Ninth ACM International Conference on Web Search and Data Mining}, WSDM '16, pp.\  457–466, New York, NY, USA, 2016. Association for Computing Machinery.
\newblock ISBN 9781450337168.
\newblock \doi{10.1145/2835776.2835804}.
\newblock URL \url{https://doi.org/10.1145/2835776.2835804}.

\bibitem[Soufiani et~al.(2012)Soufiani, Parkes, and Xia]{10.5555/2999134.2999149}
Soufiani, H.~A., Parkes, D.~C., and Xia, L.
\newblock Random utility theory for social choice.
\newblock In \emph{Proceedings of the 25th International Conference on Neural Information Processing Systems - Volume 1}, NIPS'12, pp.\  126–134, Red Hook, NY, USA, 2012. Curran Associates Inc.

\bibitem[Sui et~al.(2017)Sui, Zhuang, Burdick, and Yue]{DBLP:journals/corr/SuiZBY17}
Sui, Y., Zhuang, V., Burdick, J.~W., and Yue, Y.
\newblock Multi-dueling bandits with dependent arms.
\newblock \emph{CoRR}, abs/1705.00253, 2017.
\newblock URL \url{http://arxiv.org/abs/1705.00253}.

\bibitem[Urvoy et~al.(2013)Urvoy, Clerot, Féraud, and Naamane]{pmlr-v28-urvoy13}
Urvoy, T., Clerot, F., Féraud, R., and Naamane, S.
\newblock Generic exploration and {K}-armed voting bandits.
\newblock In Dasgupta, S. and McAllester, D. (eds.), \emph{Proceedings of the 30th International Conference on Machine Learning}, volume~28 of \emph{Proceedings of Machine Learning Research}, pp.\  91--99, Atlanta, Georgia, USA, 17--19 Jun 2013. PMLR.
\newblock URL \url{https://proceedings.mlr.press/v28/urvoy13.html}.

\bibitem[Wu \& Liu(2016)Wu and Liu]{10.5555/3157096.3157169}
Wu, H. and Liu, X.
\newblock Double thompson sampling for dueling bandits.
\newblock In \emph{Proceedings of the 30th International Conference on Neural Information Processing Systems}, NIPS'16, pp.\  649–657, Red Hook, NY, USA, 2016. Curran Associates Inc.
\newblock ISBN 9781510838819.

\bibitem[Yue \& Joachims(2009)Yue and Joachims]{10.1145/1553374.1553527}
Yue, Y. and Joachims, T.
\newblock Interactively optimizing information retrieval systems as a dueling bandits problem.
\newblock In \emph{Proceedings of the 26th Annual International Conference on Machine Learning}, ICML '09, pp.\  1201–1208, New York, NY, USA, 2009. Association for Computing Machinery.
\newblock ISBN 9781605585161.
\newblock \doi{10.1145/1553374.1553527}.
\newblock URL \url{https://doi.org/10.1145/1553374.1553527}.

\bibitem[Yue \& Joachims(2011)Yue and Joachims]{10.5555/3104482.3104513}
Yue, Y. and Joachims, T.
\newblock Beat the mean bandit.
\newblock In \emph{Proceedings of the 28th International Conference on International Conference on Machine Learning}, ICML'11, pp.\  241–248, Madison, WI, USA, 2011. Omnipress.
\newblock ISBN 9781450306195.

\bibitem[Yue et~al.(2012)Yue, Broder, Kleinberg, and Joachims]{10.1016/j.jcss.2011.12.028}
Yue, Y., Broder, J., Kleinberg, R., and Joachims, T.
\newblock The k-armed dueling bandits problem.
\newblock \emph{J. Comput. Syst. Sci.}, 78\penalty0 (5):\penalty0 1538–1556, sep 2012.
\newblock ISSN 0022-0000.
\newblock \doi{10.1016/j.jcss.2011.12.028}.
\newblock URL \url{https://doi.org/10.1016/j.jcss.2011.12.028}.

\bibitem[Zoghi et~al.(2014)Zoghi, Whiteson, Munos, and Rijke]{pmlr-v32-zoghi14}
Zoghi, M., Whiteson, S., Munos, R., and Rijke, M.
\newblock Relative upper confidence bound for the k-armed dueling bandit problem.
\newblock In Xing, E.~P. and Jebara, T. (eds.), \emph{Proceedings of the 31st International Conference on Machine Learning}, volume~32 of \emph{Proceedings of Machine Learning Research}, pp.\  10--18, Bejing, China, 22--24 Jun 2014. PMLR.
\newblock URL \url{https://proceedings.mlr.press/v32/zoghi14.html}.

\bibitem[Zoghi et~al.(2015)Zoghi, Karnin, Whiteson, and de~Rijke]{NIPS2015_9872ed9f}
Zoghi, M., Karnin, Z.~S., Whiteson, S., and de~Rijke, M.
\newblock Copeland dueling bandits.
\newblock In Cortes, C., Lawrence, N., Lee, D., Sugiyama, M., and Garnett, R. (eds.), \emph{Advances in Neural Information Processing Systems}, volume~28. Curran Associates, Inc., 2015.
\newblock URL \url{https://proceedings.neurips.cc/paper_files/paper/2015/file/9872ed9fc22fc182d371c3e9ed316094-Paper.pdf}.

\end{thebibliography}
\bibliographystyle{icml2024}

\newpage
\appendix
\onecolumn

\section{Proof of Lemma \ref{lem:ExpG}}
\LemmaExpG*
\begin{proof} 
    \begin{itemize}
        \item[] \textbf{Case $1$: $m$ is even.}
            \begin{equation}
                \label{eq:ExpGCase1Eq1}
                \bE[g(m,o_t,x_t)] = \frac{ \bE[\ind(o_t = x_t)] - \frac{(m-2)}{4(m-1)}}{\frac{m}{2(m-1)}}
            \end{equation}
            Using the construction of $\SubsetAt_t$, one can write
            \begin{align}
                \bE[\ind(o_t = x_t)] &= \sum_{i=1}^{m/2} \WinBattle_t(i \given \SubsetAt_t) \nonumber \\ 
                                    &= \frac{m}{2} \left( 2 \frac{\left(\frac{m}{2}-1\right) \pref_t(x_t, x_t) + \frac{m}{2}\pref_t(x_t, y_t)}{m(m-1)} \right) \nonumber \\
                                    &= \frac{\left( \left(\frac{m}{2}-1\right) \frac{1}{2}\right)}{m-1} + \frac{m}{2(m-1)} \pref_t(x_t, y_t) \nonumber \\
                                    &= \frac{m}{2(m-1)} \pref_t(x_t, y_t) + \frac{m-2}{4(m-1)}.
                                    \label{eq:ExpGCase1Eq2}
            \end{align} 
            In the above, the second equality follows from the definition of $\WinBattle_t(i \given \SubsetAt_t)$ and \cite{saha2018battle}[Lemma 1].
            Substituting Eq. \eqref{eq:ExpGCase1Eq2} in Eq. \eqref{eq:ExpGCase1Eq1}, we get
            \begin{equation*}
                 \bE[g(m,o_t,x_t)] = \pref_t(x_t, y_t).
            \end{equation*}
        \item[] \textbf{Case $2$: $m$ is odd.} \\
        We proceed on similar lines to Case 1.
        \begin{equation}
                \label{eq:ExpGCase2Eq1}
                \bE[g(m,o_t,x_t)] = \frac{ \bE\left[\ind(o_t = x_t)\right] - \frac{(m-1)}{4m}}{\frac{m+1}{2m}}
            \end{equation}
         Using the construction of $\SubsetAt_t$, one can write
            \begin{align}
                \bE[\ind(o_t = x_t)] &= \frac{1}{2} \sum_{i=1}^{(m-1)/2} \WinBattle_t(i \given \SubsetAt_t) + \frac{1}{2} \sum_{i=1}^{(m+1)/2} \WinBattle_t(i \given \SubsetAt_t) \nonumber \\
                &= \frac{1}{2} \ \left(\frac{m-1}{2}\right) \ \left( 2 \frac{\left(\frac{m-1}{2}-1\right) \pref_t(x_t, x_t) + \frac{m+1}{2}\pref_t(x_t, y_t)}{m(m-1)} \right) \nonumber \\
                &\qquad + \frac{1}{2} \ \left(\frac{m+1}{2}\right) \ \left( 2 \frac{\left(\frac{m+1}{2}-1\right) \pref_t(x_t, x_t) + \frac{m-1}{2}\pref_t(x_t, y_t)}{m(m-1)} \right) \nonumber \\
                &=   \left(\frac{1}{2}\right) \ \left( \frac{\left(\frac{m-1}{2}-1\right) \frac{1}{2} + \frac{m+1}{2}\pref_t(x_t, y_t)}{m} \right) \nonumber \\
                &\qquad + \left(\frac{m+1}{2}\right) \ \left( \frac{\left(\frac{m+1}{2}-1\right) \frac{1}{2} + \frac{m-1}{2}\pref_t(x_t, y_t)}{m(m-1)} \right) \nonumber \\
                &=  \left(\frac{1}{2}\right) \ \left( \frac{\left(\frac{m-3}{4}\right) + \frac{m+1}{2}\pref_t(x_t, y_t)}{m} \right) + \left(\frac{m+1}{2}\right) \ \left( \frac{\left(\frac{m-1}{4}\right) + \frac{m-1}{2}\pref_t(x_t, y_t)}{m(m-1)} \right) \nonumber \\
                & = \frac{m-3}{8m} +  \left(\frac{m+1}{2}\right)  \frac{(m-1)}{4 m (m-1)} \nonumber + \left(\frac{m+1}{4m}\right) \pref_t(x_t, y_t) + \left(\frac{m+1}{2}\right) \frac{m-1}{2 m (m-1)} \pref_t(x_t, y_t) \nonumber \\
                & = \frac{m-3}{8m} +  \frac{m+1}{8m}  +  \left(\frac{m+1}{4m}\right) \pref_t(x_t, y_t) +  \left(\frac{m+1}{4m}\right) \pref_t(x_t, y_t) \nonumber \\
                &=  \left(\frac{m+1}{2m}\right) \pref_t(x_t, y_t) + \frac{m-1}{4m}  \label{eq:ExpGCase2Eq2}.
            \end{align}
             In the above, the second equality follows from the definition of $\WinBattle_t(i \given \SubsetAt_t)$ and \cite{saha2018battle}[Lemma 1].
             Substituting Eq. \eqref{eq:ExpGCase2Eq2} in Eq. \eqref{eq:ExpGCase2Eq1}, we get
            \begin{equation*}
                 \bE[g(m,o_t,x_t)] = \pref_t(x_t, y_t).
            \end{equation*}
    \end{itemize}
\end{proof}

\section{Proof of Lemma \ref{lem:ExpScore}}
\LemmaExpScore*
\begin{proof}
    \begin{align*}
        \bE[\hat{s}_t(i)] &= \bE_{\HistoryAt_t} \left[ \frac{\ind(i=x_t)}{K\ q_t(i)} \sum_{j \in [K]} \frac{\ind(j=y_t) \ g(m,o_t,x_t) }{q_t(j)} \right]  \\
        &= \frac{1}{K} 
            \left( 
                \bE_{\HistoryAt_{t-1}}
                    \left[\bE_{(x_t, y_t, o_t)}
                        \left[
                            \frac{\ind{(i=x_t)}}{q_t(i)}
                            \sum_{j \in [K]} \frac{\ind(j=y_t) \ g(m,o_t,x_t)}{q_t(j)}
                            \given[\bigg]\HistoryAt_{t-1}  
                        \right]
                    \right]
            \right) \\
        &= \frac{1}{K} 
            \left( 
                \bE_{\HistoryAt_{t-1}}
                    \left[\bE_{x_t}
                        \left[
                            \frac{\ind{(i=x_t)}}{q_t(i)} 
                                    \sum_{j \in [K]} \bE_{y_t} 
                                    \Biggl[ 
                                        \frac{\ind(j=y_t) \ \bE_{o_t} \left[g(m,o_t,x_t) \given o_t\right]}{q_t(j)}  \given[\bigg] x_t 
                                    \Biggr]
                            \given[\Bigg]\HistoryAt_{t-1}  
                        \right]
                    \right]
            \right) \\
          &= \frac{1}{K}
            \left( 
                \bE_{\HistoryAt_{t-1}}
                \left[\bE_{x_t}
                    \left[
                        \frac{\ind{(i=x_t)}}{q_t(i)}  
                        \sum_{j \in [K]}
                        \bE_{y_t}
                        \Biggl[ 
                            \frac{\ind{(j=y_t)}  \  \pref_t(x_t, y_t) }{q_t(j)} 
                            \given[\bigg] x_t
                        \Biggr]
                        \given[\Bigg]\HistoryAt_{t-1}
                    \right]
                \right]
            \right) \\        
          &= \frac{1}{K}
            \left( 
                \bE_{\HistoryAt_{t-1}}
                \left[\bE_{x_t}
                    \left[
                        \frac{\ind{(i=x_t)}}{q_t(i)}  
                        \sum_{j \in [K]}
                        \sum_{j' \in [K]}
                        \left[
                            \frac{\ind{(j=j')}  \  \pref_t(x_t, j') \ q_t(j') }{q_t(j)} 
                        \right]
                        \given[\bigg]\HistoryAt_{t-1}
                    \right]
                \right]
            \right) \\ 
        &= \frac{1}{K}
            \left( 
                \bE_{\HistoryAt_{t-1}}
                \left[\bE_{x_t}
                    \left[
                        \frac{\ind{(i=x_t)}}{q_t(i)}  
                        \sum_{j \in [K]}
                            \pref_t(x_t, j)  
                        \given[\bigg]\HistoryAt_{t-1}
                    \right]
                \right]
            \right) \\
         &= \frac{1}{K}
            \left( 
                \bE_{\HistoryAt_{t-1}}
                \left[ \sum_{i' \in [K]}
                    \left[
                        \frac{\ind{(i=i')} q_t(i')}{q_t(i)}  
                        \sum_{j \in [K]}
                            \pref_t(i', j)  
                    \right]
                \right]
            \right) \\
         &= \frac{1}{K}
                \sum_{j \in [K]}
                \pref_t(i, j) \\
        &= s_t(i).
    \end{align*}
   In the above, the fourth equality is due to Lemma \ref{lem:ExpG}. Moreover, the fifth equality and the seventh equality use that $x_t, y_t \sim q_t$ iid with replacement.  
\end{proof}

\section{Proof of Lemma \ref{lem:GUpperBound}}
\LemmaGUpperBound*
\begin{proof}
\textbf{Case 1: $m$ is even.}
            \begin{align*}
                g(m,o_t,x_t) &\leq  \frac{ 1 - \frac{(m-2)}{4(m-1)}}{ \frac{m}{2(m-1)}} \\
                &=  \frac{ \left( \frac{4(m-1) - (m-2)}{4(m-1)}\right)}{  \frac{m}{2(m-1)}} \\
                &= \left( \frac{(3m-2)}{4 (m-1)} \right) \left( \frac{2(m-1)}{m} \right) \\
                &= \frac{(3m-2)}{2 m} 
            \end{align*}
\textbf{Case 2: $m$ is odd.}
         \begin{align*}
                g(m,o_t,x_t) &\leq  \frac{1 - \frac{(m-1)}{4m}}{\frac{m+1}{2m}} \\
                            &= \left( \frac{4m - m + 1}{4m}  \right) \left( \frac{2m}{m+1} \right) \\
                            &= \frac{3m+1}{2(m+1)}
            \end{align*}     
    
    For $m\geq 2$, 
    \begin{equation*}
        \frac{(3m-2)}{2 m} < \frac{3m+1}{2(m+1)}.
    \end{equation*}
\end{proof}

\section{Proof of Lemma \ref{lem:ScoreUpperBound}}
\LemmaScoreUpperBound*
\begin{proof}
    From the definition of $q_t$ given in \eqref{def:WeightUpdate}, it can be seen that, for all $t \in [T]$ 
    and $i \in [K]$,
    \begin{equation*}
        q_t(i) \geq \frac{\gamma}{K}.
    \end{equation*}
    Using the above along with the definition of $\hat{s}_t$ given in Eq. \eqref{def:SBordaEstimate} and Lemma \ref{lem:GUpperBound}, it can be seen that, for all $t \in [T]$ 
    and $i \in [K]$,
    \begin{equation*}
        \hat{s}_t(i) \leq \frac{(3/2)}{(K) \ (\gamma/K)^2} = \frac{3K}{2\gamma^2}.
    \end{equation*}
    Then, using $\gamma \geq  \sqrt{3\eta K/2}$ and the above inequality,  
    \begin{align*}
        \eta \hat{s}_t(i) \leq \frac{3\eta K}{2\gamma^2} = 1.
    \end{align*}
    Furthermore, $0 \leq \eta \hat{s}_t(i)$ also holds using the definition of $\hat{s}_t(i)$ given in Eq. \eqref{def:SBordaEstimate}, Lemma \ref{lem:GUpperBound} and $\eta > 0$. 
    \paragraph{}
\end{proof}
\section{Proof of Lemma \ref{lem:ExpWeightTScore}}
\LemmaExpWeightTScore*
\begin{proof}
    \begin{align*}
        \bE_{\HistoryAt_t} \left[q_t^{\transpose} \ \hat{s}_t \right] 
            &=  \bE_{\HistoryAt_t} \left[ \sum_{i=1}^{K} q_t(i) \ \hat{s}_t(i) \right] 
            = \bE_{\HistoryAt_{t-1}} \left[ \sum_{i=1}^{K} \ q_t(i) \ \bE_{x_t, y_t, o_t}\left[ \hat{s}_t(i) \given[\big] \HistoryAt_{t-1} \right] \right] 
        \\
            &= \bE_{\HistoryAt_{t-1}} \left[ \sum_{i=1}^{K} q_t(i) \ s_t(i) \right] = \bE_{\HistoryAt_{t-1}} \bigl[ \bE_{i \sim q_t} \left[ s_t(i)  \given \HistoryAt_{t-1} \right] \bigr].
    \end{align*}
    In the above, the third equality follows from the proof of Lemma \ref{lem:ExpScore}.
\end{proof}

\section{Proof of Lemma \ref{lem:ExpWeightTScoreSquared}}

\LemmaExpWeightTScoreSquared*
\begin{proof}
    \begin{alignat*}{2}
        & \bE\left[ \sum_{i=1}^{K} q_t(i) \ \hat{s}_t(i)^2 \right] \\
        &= \bE_{\HistoryAt_{t-1}} 
            \left[ \sum_{i=1}^{K} q_t(i) \ 
                \bE_{\left(x_t, y_t, a_t\right)}
                \Biggl[ \frac{\ind(i=x_t)}{K \ q_t(i)} \sum_{j \in [K]} \frac{\ind(j=y_t) \, g(m,o_t,x_t)}{q_t(j)} \given[\bigg] \HistoryAt_{t-1} \Biggr]^2
            \right] \\
          &= \frac{1}{K^2} 
          \left( \bE_{\HistoryAt_{t-1}} 
                    \left[ \sum_{i=1}^{K} \frac{q_t(i)}{q_t(i)^2} 
                    \bE_{(x_t, y_t)} \left[ \sum_{j \in [K]}  \frac{\ind(i=x_t) \, \ind(j=y_t) \, \bE_{o_t}[g^2(m,o_t,x_t) \given x_t, y_t]}{q_t(j)^2} \given[\Bigg] \HistoryAt_{t-1} \right] \right] \right) \\
         &\leq \frac{(3m+1)^2}{4(m+1)^2 K^2} 
          \left( \bE_{\HistoryAt_{t-1}} 
                    \left[ \sum_{i=1}^{K} \frac{1}{q_t(i)} 
                \left[ \sum_{j \in [K]}  \frac{  \bE_{x_t} [\ind(i=x_t)] \, \bE_{y_t}[\ind(j=y_t)]}{q_t(j)^2} \given[\Bigg] \HistoryAt_{t-1} \right] \right] \right) \qquad \text{(using Lemma \ref{lem:GUpperBound})} \\
        &= \frac{(3m+1)^2}{4(m+1)^2 K^2}  
            \left( 
                \bE_{\HistoryAt_{t-1}}
                \left[
                     \sum_{i=1}^{K} \frac{1}{q_t(i)} \left[ \sum_{j \in [K]} \frac{q_t(i) \ q_t(j)}{q_t(j)^2} \right]
                \right]
            \right)  \\
         &= \frac{(3m+1)^2}{4(m+1)^2 K^2}  
            \left( 
                \bE_{\HistoryAt_{t-1}}
                    \left[ K \sum_{j \in [K]} \frac{1}{q_t(j)} \right]
            \right) \\
        &\leq  \frac{(3m+1)^2}{4(m+1)^2 K}  
            \left( 
                \sum_{j \in [K]} \frac{1}{\gamma/K} 
            \right) 
                \qquad (\because \forall i' \in [K] \text{ and } \forall t \in [T], q_t(i') \geq \gamma/K \text{ using Eq. \eqref{def:WeightUpdate}}) \\
         &=  \frac{(3m+1)^2}{4(m+1)^2 K}  
            \left( 
               \frac{K}{\gamma/K}
            \right) \\
         &=  \frac{(3m+1)^2}{4(m+1)^2} \, \frac{K}{\gamma}.  
    \end{alignat*}
\end{proof}

\section{Proof of Lemma \ref{lem:Sampling}}
\LemmaSampling*

\begin{proof}
    \begin{align*}
         \bP\Bigl(\SubsetAt_t(j) = i\Bigr) &= \bP(\SubsetAt_t(j) = x_t) \, \bP(x_t = i) + \bP(\SubsetAt_t(j) = y_t) \, \bP(y_t = i) \\
         &=  \bP(\SubsetAt_t(j) = x_t) \, q_t(i) + \bP(\SubsetAt_t(j) = y_t) \, q_t(i) \\
         &= q_t(i).
    \end{align*}
\end{proof}

\section{Proof of Lemma \ref{lem:LowerBound}}
\LemLowerBound*
\begin{proof}
Let $r_t(\alg_{DB})$ be the instantaneous regret of $\alg_{DB}$ at round $t$. Correspondingly, let $r_t(\alg_{MB})$ be the instantaneous regret of $\alg_{MB}$ at round $t$.
\begin{align*}
    &\bE_{i_t, j_t  \overset{Unif}{\sim} [m], i_t \neq j_t} [r_t (\alg_{DB})] \\
    &= b_t(i^*) - \frac{1}{2} \left[ \bE_{i_t, j_t  \overset{Unif}{\sim} [m], i_t \neq j_t} \biggl[ b_t \Bigl(\SubsetAt_t(i_t)\Bigr) + b_t\Bigl(\SubsetAt_t(j_t)\Bigr)\biggr] \right] \\
    &=  b_t(i^*) - \frac{1}{2} \left[ \sum_{i=1}^{m} \left( \frac{b_t \Bigl(\SubsetAt_t(i)\Bigr)}{m} + \frac{\sum_{j=1, j \neq i}^{m}b_t\Bigl(\SubsetAt_t(j)\Bigr)}{m(m-1)}  \right) \right]\\
    &=   b_t(i^*) - \frac{1}{2m} \left[  \sum_{i=1}^{m} \left(b_t \Bigl(\SubsetAt_t(i)\Bigr) + \frac{\sum_{j=1}^{m}b_t\Bigl(\SubsetAt_t(j)\Bigr) - b_t\Bigl(\SubsetAt_t(i)\Bigr)}{(m-1)}  \right) \right]\\
    &=   b_t(i^*) - \frac{1}{2m} \left[ \sum_{i=1}^{m} \left(b_t \Bigl(\SubsetAt_t(i)\Bigr) - \frac{ b_t \Bigl(\SubsetAt_t(i)\Bigr)}{(m-1)}\right) + \frac{m \sum_{j=1}^{m}b_t\Bigl(\SubsetAt_t(j)\Bigr)}{(m-1)}  \right]\\
    %
    %
    &=   b_t(i^*) - \frac{1}{2m}  \left[ \frac{(m-2)}{(m-1)} \sum_{i=1}^{m} b_t \Bigl(\SubsetAt_t(i)\Bigr) + \frac{m \sum_{i=1}^{m}b_t\Bigl(\SubsetAt_t(i)\Bigr)}{(m-1)}  \right]\\
    &=   b_t(i^*) - \frac{1}{2m} \left[ \frac{\sum_{i=1}^{m} b_t \Bigl(\SubsetAt_t(i)\Bigr)}{(m-1)} \, (m -2 + m)  \right]\\
    &=   b_t(i^*) - \frac{1}{2m} \left[ \frac{\sum_{i=1}^{m} b_t \Bigl(\SubsetAt_t(i)\Bigr)}{(m-1)} \, (2m-2)  \right]\\
    &=   b_t(i^*) - \frac{1}{m} \left[ \sum_{i=1}^{m} b_t \Bigl(\SubsetAt_t(i)\Bigr)\right] \\
    & = r_t(\alg_{MB}).
\end{align*}
In the above, the second equality uses that $i_t, j_t$ are sampled uniformly at random from $[m]$ without replacement. 

Then,
\begin{align*}
    \bE[\regret_T(\alg_{DB})] = \sum_{t=1}^{T} \bE[r_t (\alg_{DB})] =  \sum_{t=1}^{T} r_t(\alg_{MB}) = \regret_T(\alg_{MB}).
\end{align*}
\end{proof}
\end{document}